\documentclass{article} % For LaTeX2e

\usepackage[accepted]{icml2021}

\usepackage{amsmath,amsfonts,bm}

\def\eqref#1{equation~\ref{#1}}
\def\1{\bm{1}}

\def\eps{{\epsilon}}

\def\rve{{\mathbf{e}}}

\def\rmD{{\mathbf{D}}}

\def\rmX{{\mathbf{X}}}
\def\rmY{{\mathbf{Y}}}

\DeclareMathAlphabet{\mathsfit}{\encodingdefault}{\sfdefault}{m}{sl}
\SetMathAlphabet{\mathsfit}{bold}{\encodingdefault}{\sfdefault}{bx}{n}

\newcommand{\R}{\mathbb{R}}

\usepackage{hyperref}
\usepackage{url}

\usepackage[utf8]{inputenc} % allow utf-8 input
\usepackage[T1]{fontenc}    % use 8-bit T1 fonts
\usepackage{url}            % simple URL typesetting
\usepackage{booktabs}       % professional-quality tables
\usepackage{amsfonts}       % blackboard math symbols
\usepackage{nicefrac}       % compact symbols for 1/2, etc.
\usepackage{microtype}      % microtypography
\usepackage{lipsum}
\usepackage[dvipsnames,table]{xcolor}
\usepackage{cleveref}
\usepackage{wrapfig}
\usepackage{subcaption}
\usepackage{booktabs}
\usepackage[ruled,vlined]{algorithm2e}
\usepackage{comment}
\usepackage{natbib}
\usepackage[title]{appendix}
\usepackage{scalerel}
\usepackage{outlines}
\usepackage{amsthm}
\usepackage{bm}
\usepackage[withAsciilist=true]{typed-checklist}
\usepackage{todonotes}
\usepackage[title]{appendix}
\definecolor{nicered}{HTML}{c73912}
\definecolor{niceblue}{HTML}{2F80ED}
\definecolor{nicegreen}{HTML}{0e9e36}
\newcommand{\red}{\textcolor{nicered}}

\newcommand{\KL}[2]{D_{\mathrm{KL}} \bigl( #1 ~||~ #2 \bigr)}
\DeclareMathOperator*{\E}{\scaleobj{1.1}{\mathbb{E}}}

\renewcommand{\eps}{\varepsilon}

\hypersetup{
    colorlinks = true,
    citecolor = .,
    linkcolor = .,
    urlcolor = niceblue
}

\newtheorem{theorem}{Theorem}

\newtheorem{definition}[theorem]{Definition}

\definecolor{col_a}{HTML}{FFE5E6}
\definecolor{col_b}{HTML}{E5F3FF}
\definecolor{col_c}{HTML}{E6FFE5}

\newcolumntype{a}{>{\columncolor{col_a}} r}
\newcolumntype{b}{>{\columncolor{col_b}} r}
\newcolumntype{c}{>{\columncolor{col_c}} r}

\icmltitlerunning{Evaluating representations by the complexity of learning low-loss predictors}

\begin{document}

\twocolumn[
\icmltitle{Evaluating representations by the complexity of learning low-loss predictors}



\icmlsetsymbol{equal}{*}

\begin{icmlauthorlist}
\icmlauthor{William F. Whitney}{nyu}
\icmlauthor{Min Jae Song}{nyu}
\icmlauthor{David Brandfonbrener}{nyu}
\icmlauthor{Jaan Altosaar}{col}
\icmlauthor{Kyunghyun Cho}{nyu}
\end{icmlauthorlist}

\icmlaffiliation{nyu}{Courant Institute, New York University}
\icmlaffiliation{col}{Columbia University}

\icmlcorrespondingauthor{Will Whitney}{wwhitney@cs.nyu.edu}
% \icmlcorrespondingauthor{Eee Pppp}{ep@eden.co.uk}

\icmlkeywords{Machine Learning, ICML, representation learning, representation evaluation, mutual information, minimum description length}

\vskip 0.3in
]



\printAffiliationsAndNotice{\icmlEqualContribution} % otherwise use the standard text.

% \maketitle

% TLDR: Good representations allow simpler predictors to achieve low loss

\begin{abstract}
We consider the problem of evaluating representations of data for use in solving a downstream task.
We propose to measure the quality of a representation by the complexity of learning a predictor on top of the representation that achieves low loss on a task of interest.
To this end, we introduce two measures: surplus description length (SDL) and $\eps$ sample complexity ($\eps$SC).
To compare our methods to prior work, we also present a framework based on plotting the validation loss versus evaluation dataset size (the ``loss-data'' curve).
Existing measures, such as mutual information and minimum description length, correspond to slices and integrals along the data axis of the loss-data curve, while ours correspond to slices and integrals along the loss axis.
This analysis shows that prior methods measure properties of an evaluation dataset of a specified size, whereas our methods measure properties of a predictor with a specified loss.
We conclude with experiments on real data to compare the behavior of these methods over datasets of varying size.
%along with a high performance \href{https://github.com/willwhitney/reprieve}{open-source library} for representation evaluation.

% A high performance open source implementation of the measures discussed here is available at
% \url{https://github.com/willwhitney/reprieve}.

% demonstrate the sensitivity of existing methods to the dataset size.
% We theoretically and empirically show that
% Theoretically and empirically we show that these methods
% by choosing where to slice or integrate the loss-data curve, each of these existing methods
% Theoretically and empirically we show that by choosing where to slice or integrate the loss-data curve, each of these existing methods makes an {\color{blue} implicit tradeoff between complexity of learning and accuracy}, yielding different orderings of representations as the dataset size changes.
% Our method  learning explicit, making it intuitive to use by enabling the user to specify a target accuracy and removing the confounding effect of dataset size.
\end{abstract}

\begin{comment}

\begin{AsciiList}[subsection*,TaskList,TaskList]{-,*,+}
- high-level goals
    * done: information about setting $\eps$
    * done: computational cost
    * done: emphasize motivation
    * done: emphasize that no premature evaluation is good
    * open: add some larger-scale experiments as proof of concept
- lower-level changes
    * done: switch to talking about the size of the "evaluation set" $\leftarrow$ \red{MJ: what does this mean?}
        + done: make sure this is clearly defined up front
    * done: emphasize that we view dataset size as a changeable quantity.
    * open: experiments on ImageNet or CIFAR
        + open: SimCLR
        + open: \href{https://github.com/facebookresearch/moco}{MoCo}
        + open: pretrained model from torchvision
        + open: randomly initialized resnet
    * done: sharpen the experiments section with some takeaways
    * open: rewrite abstract with new language
    * done: cut back to size
\end{AsciiList}

\end{comment}

\section{Introduction}

One of the first steps in building a machine learning system is selecting a representation of data.
Whereas classical machine learning pipelines often begin with feature engineering, the advent of deep learning has led many to argue for pure end-to-end learning where the deep network constructs the features \citep{lecun2015deep}.
However, huge strides in unsupervised learning \citep{Hnaff2020DataEfficientIR,Chen2020ASF, he2019momentum, Oord2018RepresentationLW, bachman2019learning, devlin2018bert, liu2019roberta, raffel2019exploring, gpt3} have led to a reversal of this trend in the past two years, with common wisdom now recommending that the design of most systems start from a pretrained representation.
With this boom in representation learning techniques, practitioners and representation researchers alike have the question: Which representation is best for my task?

\begin{figure}[h]
\centering
\includegraphics[width=0.48\textwidth]{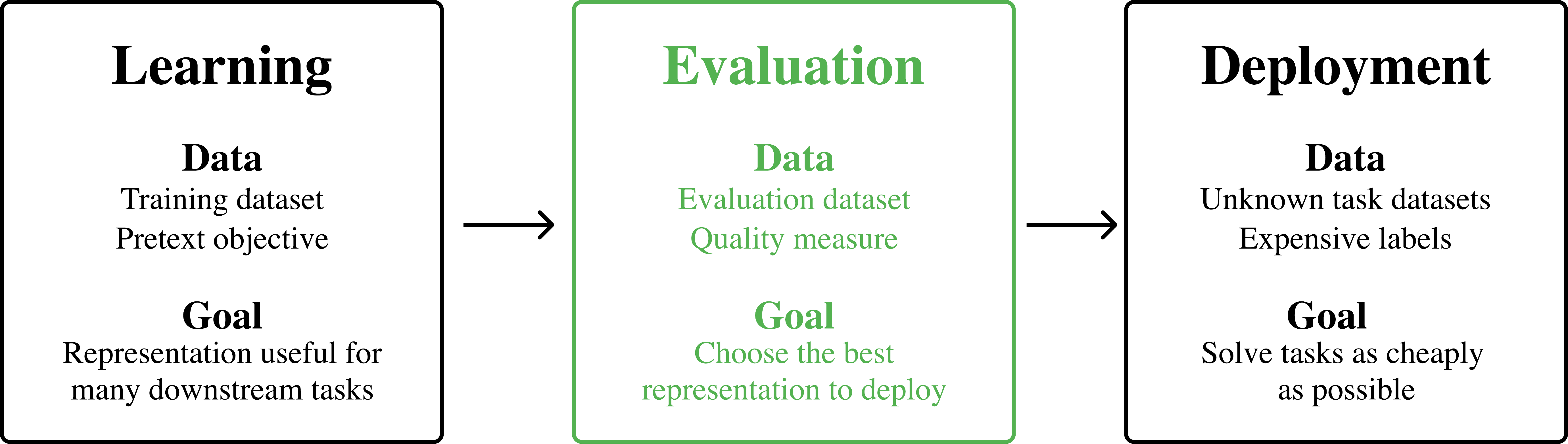}
\caption{The representation learning pipeline.}
\label{fig:representation_pipeline}
\end{figure}

This question exists as the middle step of the representation learning pipeline shown in \Cref{fig:representation_pipeline}.
The first step is representation learning, which consists of training a representation function on a training set using a pretext objective, which may be supervised or unsupervised.
The second step, which this paper considers, is representation evaluation.
In this step, one uses a measure of representation quality and a labeled \emph{evaluation dataset} to see how well the representation performs.
The final step is deployment, in which the practitioner or researcher puts the learned representation to use.
Deployment could involve using the representation on a stream of user-provided data to solve a variety of end tasks \citep{embedtheworld}, or simply releasing the trained weights of the representation function for general use.
In the same way that BERT \citep{devlin2018bert} representations have been applied to a whole host of problems, the task or amount of data available in deployment might differ from the evaluation phase.

% Answering this question in a principled manner requires deciding how to define ``task'' and ``best''.

% In this paper, we take the task to be finding a predictor that achieves low expected risk on a downstream supervised learning problem.
We take the position that the best representation is the one which allows for the most \emph{efficient} learning of a predictor to solve the task.
We will measure efficiency in terms of either number of samples or information about the optimal predictor contained in the samples.
This position is motivated by practical concerns; the more labels that are needed to solve a task in the deployment phase, the more expensive to use and the less widely applicable a representation will be.
To date the field has lacked clearly defined and motivated tools for analyzing the complexity of learning with a given representation.
This work seeks to
% elucidate the subleties of representation evaluation and
provide those tools for the representation learning community.

\begin{figure*}[t]
\begin{subfigure}[b]{.33\textwidth}
  \centering
  % include first image
  \includegraphics[height=130px]{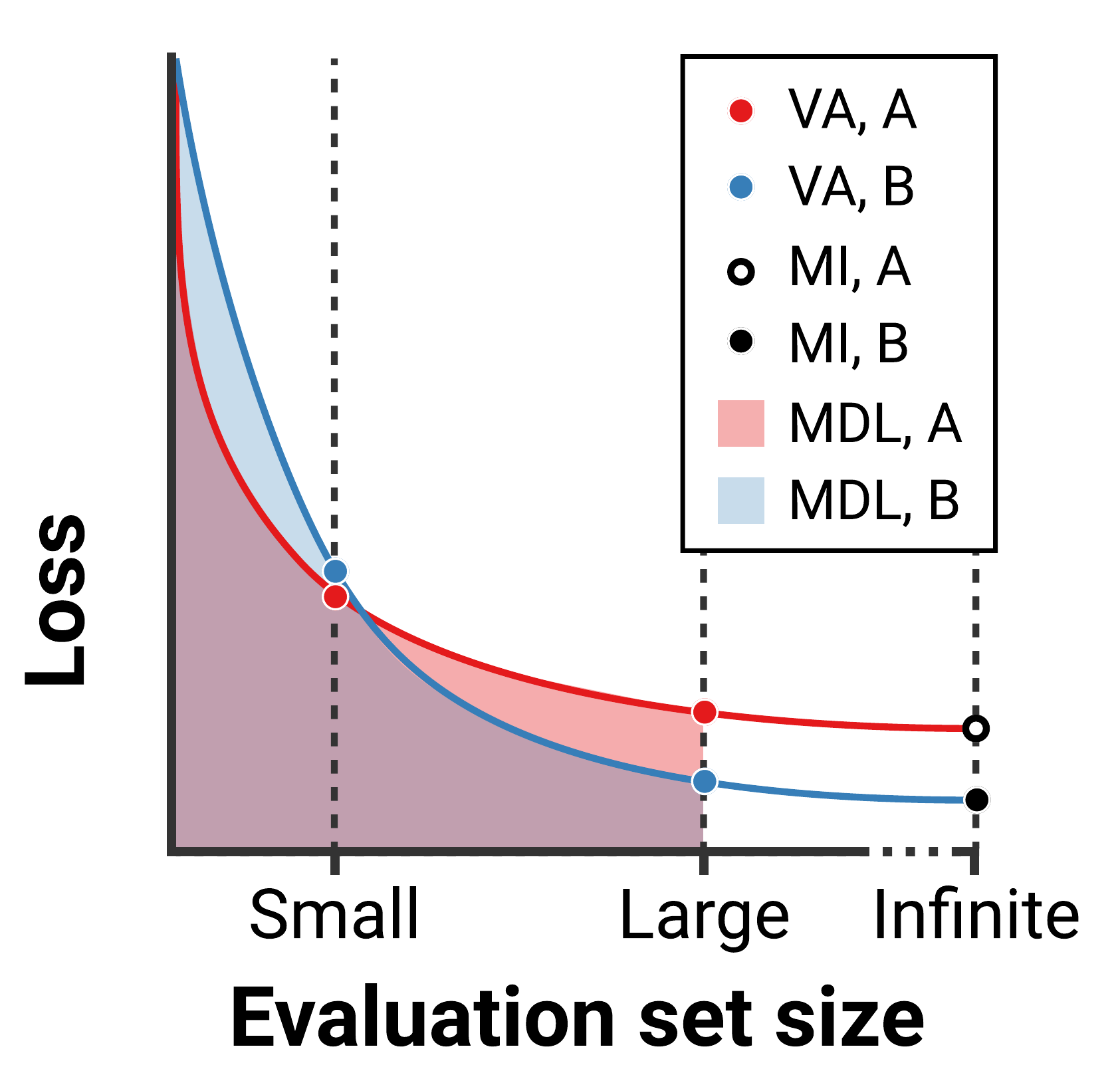}
  \caption{Existing measures}
  \label{fig:cartoon_left}
\end{subfigure}
\begin{subfigure}[b]{.32\textwidth}
  \centering
  % include second image
  \includegraphics[height=130px]{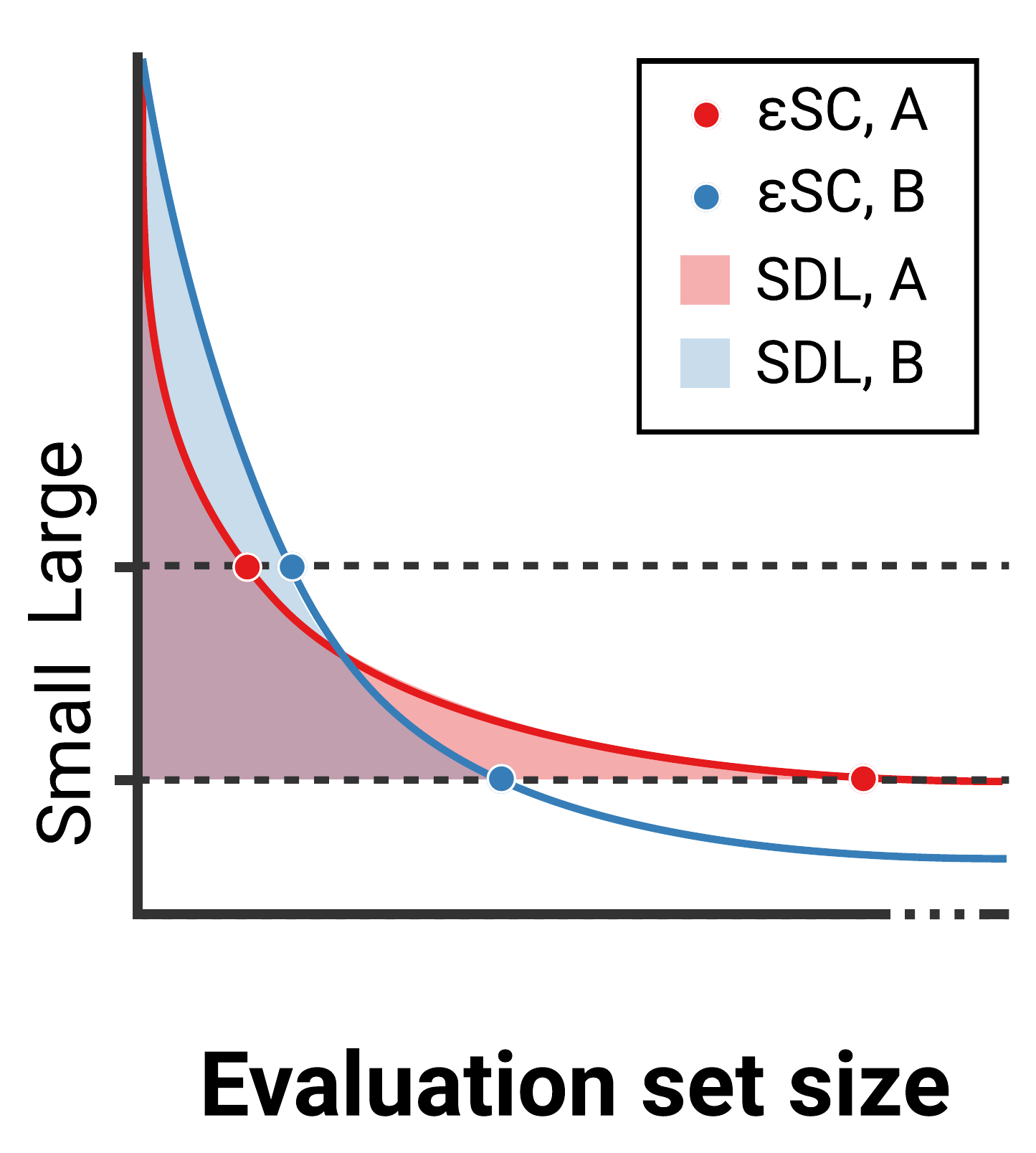}
  \caption{Proposed measures}
  \label{fig:cartoon_right}
\end{subfigure}
\begin{subfigure}[b]{.32\textwidth}
  \centering
  % include second image
  \includegraphics[height=130px]{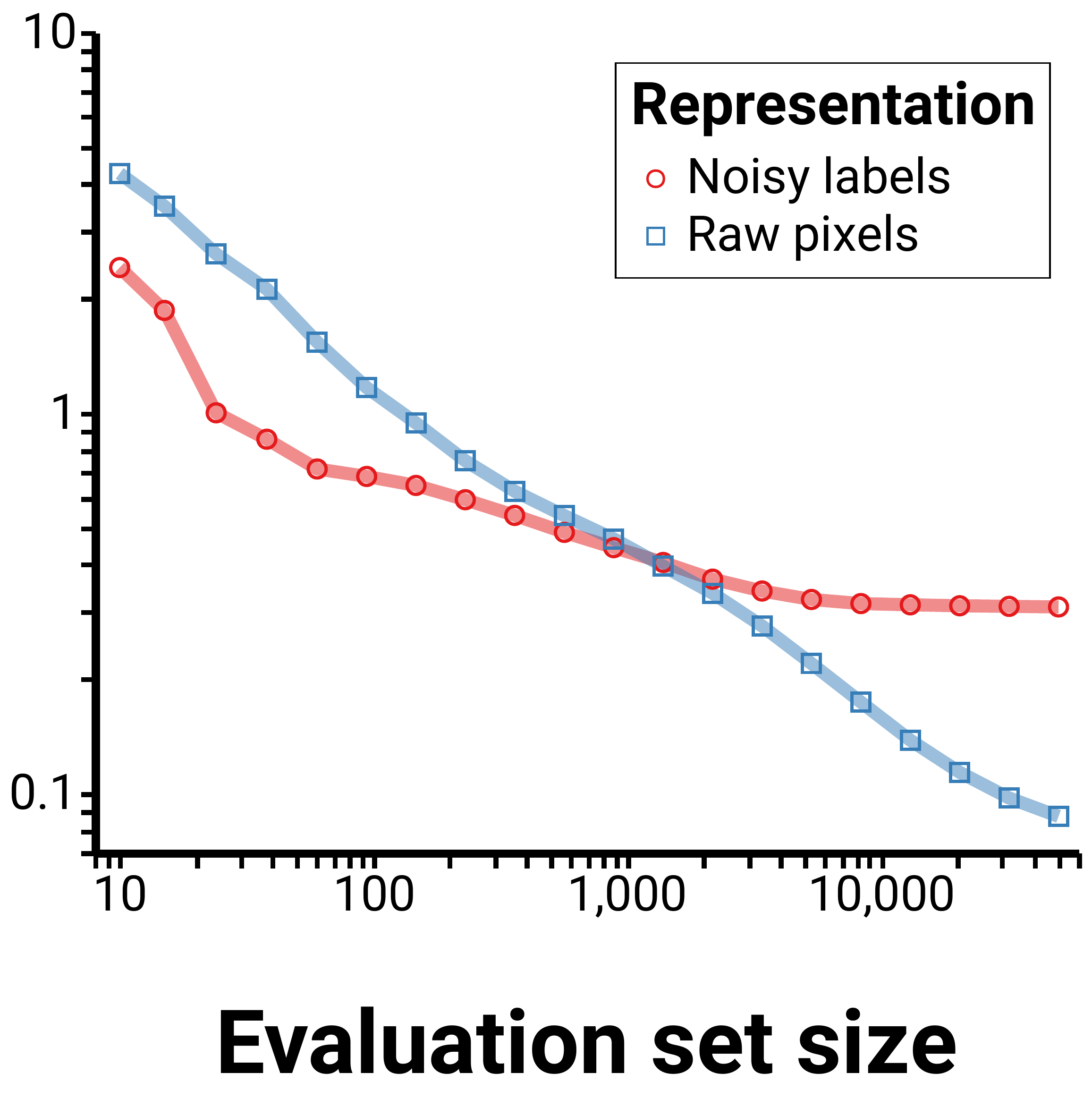}
  \caption{Illustrative experiment}
  \label{fig:noisygt}
\end{subfigure}
\caption{
Each measure for evaluating representation quality is a simple function of the ``loss-data'' curve shown here, which plots validation loss of a probe against evaluation dataset size.
\textbf{Left:} Validation accuracy (VA), mutual information (MI), and minimum description length (MDL) measure properties of a given evaluation dataset, with VA measuring the loss at a finite amount of evaluation data, MI measuring it at infinity, and MDL integrating it from zero to $n$.
This dependence on evaluation dataset size can lead to misleading conclusions as the amount of available data changes.
\textbf{Middle:} Our proposed methods instead measure the complexity of learning a predictor with a particular loss tolerance.
$ \eps$ sample complexity ($\eps$SC) measures the number of samples required to reach that loss tolerance, while surplus description length (SDL) integrates the surplus loss incurred above that tolerance.
Neither depends on the evaluation dataset size.
\textbf{Right:} A simple example task which illustrates the issue.
One representation, which consists of noisy labels, allows quick learning, while the other supports low loss in the limit of data.
Evaluating either representation at a particular evaluation dataset size risks drawing the wrong conclusion.
}
\label{fig:fig}
\end{figure*}

% This position is motivated by practical concerns; a standard ML application involves an iterative process between collecting data and learning models, and the greatest cost in many applications is data collection \red{(citation needed)}.
% The promise of representation learning is that it can enable solving new tasks while spending less effort collecting data.
% \red{MJ: The last sentence is a bit unclear to me. Alternative: we view dataset size as a quantity subject to change. This perspective is in contrast with previous evaluation methods which tend to implicitly assume that dataset size is given and fixed.}

We build on a substantial and growing body of literature that attempts to answer the question of which representation is best.
Simple, traditional means of evaluating representations, such as the validation accuracy of linear probes \citep{ettinger2016probing,Shi2016String,Alain2016Understanding}, have been widely criticized \citep{Hnaff2020DataEfficientIR,resnick2019probing}.
Instead, researchers have taken up a variety of alternatives such as the validation accuracy (VA) of nonlinear probes \citep{Conneau2018Cram,Hnaff2020DataEfficientIR}, mutual information (MI) between representations and labels \citep{bachman2019learning, Pimentel2020InformationTheoreticPF}, and minimum description length (MDL) of the labels conditioned on the representations \citep{Blier2018TheDL,Yogatama2019LinguisticIntel,Voita2020InformationTheoreticPW}.

We find that these methods all have clear limitations.
% which stem
% VA and MDL measure properties of the data, with VA measuring as the attainable accuracy using a specific amount of data and MDL measuring the compressibility of a fixed number of labels given the observations.
As can be seen in \Cref{fig:fig}, VA and MDL are liable to choose different representations for the same task when given evaluation datasets of different sizes.
Instead we want an evaluation measure which depends on the data \emph{distribution}, not a particular evaluation dataset sample or evaluation dataset size.
Furthermore, VA and MDL lack a predefined notion of success in solving a task.
In combination with small evaluation datasets, these measures may lead to premature evaluation by producing a judgement even when there is not enough data to solve the task or meaningfully distinguish one representation from another.
Meanwhile, MI measures the lowest loss achievable by any predictor irrespective of the number of samples required to learn it or the computational cost to compute it.
None of these existing techniques measure the improved data efficiency that a good representation can yield, despite this being one of the primary applications for representation learning.
% However, we note that these measures might still be appropriate for selecting representations in settings where .

% Even when none of the representations under consideration enable solving the task with the given amount of data,

% This means that a practitioner who evaluates two representations using a small dataset will choose the one
% Evaluating with a too-small dataset can lead a practitioner to make a premature choice of representation even when none have solved the task.

To eliminate these issues, we propose two measures of representation quality.
In both of our measures, the user specifies a tolerance $\eps$ so that a population loss of less than $\eps$ qualifies as solving the task.
Then the measure computes the cost of learning a predictor which achieves that loss.
% For example, a researcher might want to build representations that enable a learned model to perform at state of the art for ImageNet from pixels while using as few labels as possible, as in the contrastive representation learning work of \citet{Hnaff2020DataEfficientIR}.
% This allows them to compare the complexity of learning a specific function with  different representations.
% That is, the user specifies a target function and these measures compute how complex it is to learn.
% i.e. approximating the optimal predictor to low error.
The first measure is the \emph{surplus description length} (SDL) which modifies the MDL to measure the complexity of learning an $\eps$-loss predictor rather than computing the complexity of the labels in the evaluation dataset.
The second is the \emph{$\eps$-sample complexity} ($\eps$SC) which measures the sample complexity of learning an $\eps$-loss predictor.
These measures resolve the issues with prior work and provide tools for researchers and practitioners to evaluate the extent to which a learned representation can improve data efficiency.
Furthermore, they formalize existing research challenges for learning representations which allow state of the art performance while using as few labels as possible (e.g.~\citet{Hnaff2020DataEfficientIR}).

To facilitate our analysis, we also propose a framework called the \emph{loss-data framework}, illustrated in \Cref{fig:fig}, that plots the validation loss against the evaluation dataset size \citep{Talmor2019oLMpics, Yogatama2019LinguisticIntel, Voita2020InformationTheoreticPW}.
This framework simplifies comparisons between measures.
Prior work measures integrals (MDL) and slices (VA and MI) along the data axis.
Our work proposes instead measuring integrals (SDL) and slices ($\eps$SC) along the loss axis.
This illustrates how prior work makes tacit choices about the function to learn based on the choice of evaluation dataset size.
Our work instead makes an explicit, interpretable choice of what function to learn via the threshold $\eps$ and measures the complexity of learning such a function.
We experimentally investigate the behavior of these methods, illustrating the sensitivity of VA and MDL, and the robustness of SDL and $\eps$SC, to evaluation dataset size.

\paragraph{Efficient implementation.} To enable reproducible representation evaluation for representation researchers, we have developed a highly optimized open source Python package (see supplementary materials).
%at \url{https://github.com/willwhitney/reprieve}.
This package enables construction of loss-data curves with arbitrary representations and datasets and is library-agnostic, supporting representations and learning algorithms implemented in any Python ML library.
By leveraging the JAX library \citep{jax2018github} to parallelize the training of probes on a single accelerator, our package constructs loss-data curves in around two minutes on one GPU.

\section{The loss-data framework for representation evaluation}
% \begin{AsciiList}[itemize, itemize, itemize]{-,*,+}
% - intro para
% - setting:
%     * representations held fixed
%     * single target task
%     * applying a probe to inspect the representation
%     * goal is typically to take away something larger; "is this a good representation for tasks like this" ($\rightarrow$ deployment), not "which method gives the best numbers on exactly this dataset" ($\rightarrow$ imagenet)
%         + this consideration of uncertainty is potentially valuable to our exposition
%         + e.g. if you don't know how much data you'll have, knowing that $\phi_1$ takes fewer samples than $\phi_2$ is useful
%         + maybe would be also (more?) useful to give a breakdown: $\phi_1$ is best for $0...M$ samples, then $\phi_2$
%     * assume large probes; able to recover true $p(\rmY \mid \rmX)$ given enough data
% - we study the relationship between training set size and test loss
% - this viewpoint connects existing methods, which are specific realizations
% \end{AsciiList}

% - goal is typically to take away something larger; "is this a good representation for tasks like this" ($\rightarrow$ deployment), not "which method gives the best numbers on exactly this dataset" ($\rightarrow$ imagenet)

In this section we formally present the representation evaluation problem, define our loss-data framework, and show how prior work fits into the framework.

\paragraph{Notation.}
% \blue{switch this to talking about representations $\Phi(X)$ instead of tasks}
% We consider \emph{tasks} as joint distributions over observations and labels $(\rmX, \rmY)$.
% Finite datasets of size $N$ are $(X^N, Y^N) = \{(x_i, y_i)\}_{i=1}^N$.
% When we evaluate multiple representations of the same data, we consider them separate tasks which share the same labels.
% Comparing the suitability of two representations then corresponds to comparing the difficulty of their corresponding tasks.
% \red{TODO: look up how others have phrased this. I feel like I should somehow make the \emph{task} consist of just the labels?}
We use bold letters to denote random variables.
A supervised learning problem is defined by a joint distribution $ \mathcal{D}$ over
observations and labels $(\rmX, \rmY)$ in the sample space $ \mathcal{X} \times \mathcal{Y}$ with density denoted by $ p$.
%We will call $\rmX$ is the ``raw'' representation of the data.
%We let $ \rmX$ take values in a space $ \mathcal{X}$ and $\rmY$ take values in $ \mathcal{Y}$. We assume that $ \mathcal{D} $ admits a density $ p(x,y)$.
Let the random variable $\rmD^n$ be a sample of $n$ i.i.d. $(\rmX, \rmY)$ pairs, realized by $D^n = (X^n, Y^n) = \{(x_i, y_i)\}_{i=1}^n$.
This is the evaluation dataset.
Let $ \mathcal{R}$ denote a representation space and $\phi: \mathcal{X} \to \mathcal{R}$ a representation function.
% Given a predictive distribution $\hat{p}$, we define the population loss $\mathcal{L}(\hat{p}) = \E_{(\rmX, \rmY)} - \log \hat{p}(\rmY \mid \rmX)$. \todo{maybe define a $D^n_\phi$ or something}
The methods we consider all use parametric probes, which are neural networks $\hat{p}_\theta: \mathcal{R} \to P(\mathcal{Y})$ parameterized by $ \theta\in \R^d$ that are trained on $D^n$ to estimate the conditional distribution $p(y \mid x)$.
We often abstract away the details of learning the probe by simply referring to an algorithm $\mathcal{A}$ which returns a predictor: $\hat{p} = \mathcal{A}(\phi(D^n))$. Abusing notation, we denote the composition of $ \mathcal{A}$ with $\phi$ by $\mathcal{A}_\phi$.
Define the population loss and the expected population loss for $\hat{p} = \mathcal{A}_\phi(D^n)$, respectively as
\begin{align}
    L(\mathcal{A}_\phi, D^n) &= \E_{(\rmX, \rmY)} - \log \hat{p}(\rmY \mid \rmX) \label{eq:loss_data} \\
    L(\mathcal{A}_\phi, n) &= \E_{\rmD^n} L(\mathcal{A}_\phi, \rmD^n). %\label{eq:loss_n}
\end{align}
The expected population loss averages over evaluation dataset samples, removing the variance that comes from using some particular evaluation dataset to train a probe.
In this section we will focus on population quantities, but note that any algorithmic implementation must replace these by their empirical counterparts.
% Since these quantities are difficult to estimate from data, we will often use an empirical estimate calculated on a validation set $ D_{\mathrm{val}} $ of $ m $ i.i.d. $ (\rmX, \rmY)$ pairs:
% \begin{align}
%     \widehat L(\mathcal{A}, D^n) = \sum_{i=1}^m -\log \hat p(y_i|x_i)
% \end{align}

\paragraph{The representation evaluation problem.} The representation evaluation problem asks us to define a real-valued measurement of the quality of a representation $ \phi$ for solving solving the task defined by $(\rmX, \rmY)$. Explicitly, each method defines a real-valued function $ m(\phi, \mathcal{D}, \mathcal{A}, \Psi)$ of a representation $\phi$, data distribution $ \mathcal{D}$, probing algorithm $ \mathcal{A}$, and some method-specific set of hyperparameters $ \Psi $. By convention, smaller values of the measure $ m $ correspond to better representations.
Defining such a measurement allows us to compare different representations.

% \begin{align}
%     \lim_{n \to \infty} \mathcal{A}(D^n) = p(\rmY \mid \rmX) \label{eq:convergence_to_conditional}
% \end{align}
% where $p(\rmY \mid \rmX)$ is the true conditional distribution, and

% \paragraph{Setting}
% Our objective is to analyze the suitability of a representation for learning to solve a given single downstream task; as demonstrated by \citet{Voita2020InformationTheoreticPW}, this can be interpreted as asking ``What does this representation $\phi(\rmX)$ know about $\rmY$?''

\subsection{Defining the loss-data framework.}
The loss-data framework is a lens through which we contrast different measures  of representation quality. The key idea, demonstrated in \Cref{fig:fig}, is to plot the loss $ L(\mathcal{A}_\phi, n) $ against the evaluation dataset size $ n$.
Explicitly, at each $ n$, we train a probing algorithm $ \mathcal{A}$ using a representation $ \phi$ to produce a predictor $ \hat p$, and then plot the loss of $ \hat p$ against $ n$.
Similar analysis has appeared in \citet{Voita2020InformationTheoreticPW, Yogatama2019LinguisticIntel, Talmor2019oLMpics}.
We can represent each of the prior measures as points on the curve at fixed $ x $ (VA, MI) or integrals of the curve along the $ x$-axis (MDL).
Our measures correspond to evaluating points at fixed $ y $ ($\eps$SC) and integrals along the $ y$-axis (SDL).

\subsection{Existing methods in the loss-data framework} \label{sec:existing_methods}

\paragraph{Nonlinear probes with limited data.}
% \blue{calling the issue with standard probes "hyperparameter dependence"}

% \blue{maybe switch this to just talk about the issues from Voita: accuracy doesn't distinguish control tasks, has hyperparameter dependence}

A simple strategy for evaluating representations is to choose a probe architecture and train it on a limited amount of data from the task and representation of interest \citep{Hnaff2020DataEfficientIR, Zhang2018LanguageMT}.
Each representation is typically scored by its validation accuracy, leading us to call this the validation accuracy (VA) measure.
% The use of a single dataset sample
% Each representation then gets a score corresponding to the validation accuracy on the task, with higher accuracies indicating better representations.
This method can be interpreted in our framework by replacing the validation accuracy with the validation loss and taking an expectation over draws of evaluation datasets of size $n$.
On the loss-data curve, this measure corresponds to evaluation at $x=n$, so that
\begin{align}
    m_{\mathrm{VA}}(\phi, \mathcal{D}, \mathcal{A}, n) = L(\mathcal{A}_\phi, n).
\end{align}

% defining a probe algorithm $\mathcal{A}$ and running $\mathcal{A}$ on a (potentially artificially) limited amount of data in dataset $D^n$ to get a predictive distribution $\hat{p} = \mathcal{A}(D^n)$.
% Each representation then gets a score

% The  simplest strategy for evaluating representations is with validation accuracy: choose a model family and train it on the task of interest, learning one predictor for each representation, and compare their validation accuracies.
% Common variants of this method use linear models \citep{Oord2018RepresentationLW,Hnaff2019DataEfficientIR,Chen2020ASF} or large models, but with limited data \citep{Hnaff2019DataEfficientIR}\todo{add citations to both  bins, including NLP}.

% This method is appealingly simple, but in practice produces results which  are highly dependent on the choice of model class, dataset size, and other hyperparameters.
% For example, \citet{Hnaff2019DataEfficientIR} find that linear probes and limited-data deep probes yield different conclusions.
% We call this issue \emph{hyperparameter dependence}. (\red{add other references for the issues with this method})

\paragraph{Mutual information.}
Mutual information (MI) between a representation $\phi(\rmX)$ and targets $\rmY$ is another often-proposed metric for learning and evaluating representations \citep{Pimentel2020InformationTheoreticPF, bachman2019learning}.
%\todo{cite: infoNCE, CPC/v2, simCLR, etc — look at \citet{Song2020UnderstandingTL,Tschannen2020OnMI}}.
In terms of entropy, mutual information is equivalent to the information gain about $\rmY$ from knowing $\phi(\rmX)$:
\begin{align}
    I(\phi(\rmX); \rmY) = H(\rmY) - H(\rmY \mid \phi(\rmX)).
\end{align}
In general mutual information is intractable to estimate for high-dimensional or continuous-valued variables \citep{McAllester2018FormalLO}, and a common approach is to use a very expressive model for $\hat{p}$ and maximize a variational lower bound:
\begin{align}
    I(\phi(\rmX); \rmY) &\ge H(\rmY) + \E_{(\rmX, \rmY)} \log \hat{p}(\rmY \mid \phi(\rmX)).
    %&\approx H(\rmY) + \frac{1}{n} \sum_{i=1}^n \log \hat{p}_\theta(y_i \phi(\mid x_i))
\end{align}
Since $H(\rmY)$ is not a function of the parameters, maximizing the lower bound is equivalent to minimizing the negative log-likelihood.
Moreover, if we assume that $ \hat p$ is expressive enough to represent $ p$ and take $ n \to \infty$, this inequality becomes tight.
%For this estimate to become tight relies on the availability of infinite data and the assumption given in \cref{eq:convergence_to_conditional}.
As such, MI estimation can be seen a special case of nonlinear probes as described above, where instead of choosing some particular setting of $n$ we push it to infinity. We formally define the mutual information measure of a representation as
%By this mutual information probing method, a representation is said to be good if it has high mutual information with the labels of the task of interest.
\begin{align}
    m_{\mathrm{MI}}(\phi, \mathcal{D}, \mathcal{A}) = \lim_{n\to \infty} L(\mathcal{A}_\phi, n).
\end{align}
A decrease in this measure reflects an increase in the mutual information.
On the loss-data curve, this corresponds to evaluation at $ x= \infty$.

\paragraph{Minimum description length.}
Recent studies \citep{Yogatama2019LinguisticIntel,Voita2020InformationTheoreticPW} propose using the Minimum Description Length (MDL) principle \citep{Rissanen1978ModelingBS,Grnwald2004ATI} to evaluate representations.
These works use an online or prequential code \citep{Blier2018TheDL} to encode the labels given the representations. The codelength $ \ell$ of $ Y^n $ given $ \phi(X^n) $ is then defined as
\begin{align}
    \ell(Y^n \mid \phi(X^n)) = - \sum_{i=1}^n \log \hat{p}_i(y_{i} \mid \phi(x_{i})),
\end{align}
where $\hat{p}_i$ is the output of running a pre-specified algorithm $\mathcal{A}$ on the evaluation dataset up to element $i$: $\hat{p}_i = \mathcal{A}_\phi(X^n_{1:i}, Y^n_{1:i})$.
This measure can exhibit large variance on small evaluation datasets, especially since it is sensitive to the (random) order in which the examples are presented.
%With this method a representation $\phi$ is good if $\ell(Y^n \mid \phi(X^n))$ is small.
We remove this variance by taking an expectation over the sampled evaluation datasets for each $i$ and define a population variant of the MDL measure~\citep{Voita2020InformationTheoreticPW} as
\begin{align} \label{eq:mdl_expected}
    m_{\mathrm{MDL}}(\phi, \mathcal{D}, \mathcal{A},n) = \E \Big[ \ell(\rmY^n \mid \phi(\rmX^n)) \Big] = \sum_{i=1}^n L(\mathcal{A}, i).
\end{align}
Thus, $m_\mathrm{MDL}$ measures the area under the loss-data curve on the interval $x \in [0, n]$.

\section{Limitations of existing methods}

Each of the prior methods, VA, MDL, and MI, have limitations that we attempt to solve with our methods. In this section we present these limitations.
% In \Cref{sec:counter_ex}, we describe a toy example which demonstrates why evaluation metrics that depend on the evaluation dataset size like VA and MDL can be misleading.
% Then in \Cref{sec:mutual_info} we argue that MI, which does not depend on the evaluation dataset size, can be misleading as well since it is insensitive to the quality of the representation.
% Finally, in \Cref{sec:predefined_success}, we observe that because prior measures have no predefined notion of success,  their results are subject to overinterpretation  when  the evaluation dataset is  small or the representations are poor.

\subsection{Sensitivity to evaluation set size in VA and MDL}
As seen in \Cref{sec:existing_methods}, the representation quality measures of VA and MDL
%given by the learning procedure $A$ and training set $D^n$
both depend on $n$, the size of the evaluation dataset.
Because of this dependence, the ranking of representations given by these evaluation metrics can change as $n$ increases.
Choosing to deploy one representation rather than another by comparing these metrics at arbitrary $n$ may lead to premature decisions in the machine learning pipeline since a larger  evaluation dataset could give a different ordering.

% As a result, favoring one representation over others by comparing these metrics at arbitrary $n$ may lead to premature decisions in the machine learning pipeline since evaluating on a larger evaluation dataset could give a different ordering.

\paragraph{A theoretical example.}
\label{sec:counter_ex}
Let $s \in \{0,1\}^d$ be a fixed binary vector and consider a data generation process where the $\{0,1\}$ label of a data point is given by the parity on $s$, i.e., $y_i = \langle x_i, s \rangle \bmod{2}$ where $y_i \in \{0,1\}$ and $x_i \in \{0,1\}^d$. Let $Y^n = \{y_i\}_{i=1}^n$ be the given labels and consider the following two representations: (1) Noisy label: $z_i = \langle x_i, s \rangle + e_i \bmod{2}$, where $e_i \in \{0,1\}$ is a random bit with bias $\alpha < 1/2$, and (2) Raw data: $x_i$.

For the noisy label representation, guessing $y_i = z_i$ achieves validation accuracy of $1-\alpha$ for any $n$, which, is information-theoretically optimal. On the other hand, the raw data representation will achieve perfect validation accuracy once the evaluation dataset contains $d$ linearly independent $x_i$'s. In this case, Gaussian elimination will exactly recover $s$. The probability that a set of $n > d$ random vectors in $\{0,1\}^d$ does not contain $d$ linearly independent vectors decreases exponentially in $n-d$. Hence, the expected validation accuracy for $n$ sufficiently larger than $d$ will be exponentially close to 1. As a result, the representation ranking given by validation accuracy and description length favors the noisy label representation when $n \ll d$, but the raw data representation will be much better in these metrics when $n \gg d$. This can be misleading.
Although this is a concocted example for illustration purposes, our experiments in \Cref{sec:experiments} validate that dependence of representation rankings on $n$ does occur in practice.

\subsection{Insensitivity to representation quality  \& computational complexity in MI}
\label{sec:mutual_info}
MI considers the lowest validation loss achievable with the given representation and ignores any concerns about statistical or computational complexity of achieving such accuracy.
This leads to some counterintuitive properties which make MI an undesirable metric:
\begin{enumerate}
    \item MI is insensitive to statistical complexity. Two random variables which are perfectly predictive of one another have maximal MI, though their relationship may be sufficiently complex that it requires exponentially many samples to verify~\citep{McAllester2018FormalLO}.
    \item MI is insensitive to computational complexity. For example, the mutual information between an intercepted encrypted message and the enemy's plan is high~\citep{Shannon1948TheMT, Xu2020ATheory}, despite the extreme computational cost required to break the encryption.
    \item MI is insensitive to representation. By the data processing inequality \citep{Cover2006ElementsOI}, \emph{any} $\phi$ applied to $\rmX$ can only decrease its mutual information with $\rmY$; no matter the query, MI always reports that the raw data is at least as good as the best representation.
\end{enumerate}
%As a result, we believe that in most settings MI is an undesirable metric for evaluating representations.

\subsection{Lack of a predefined notion of success}
All three prior methods lack a predefined notion of successfully solving a task and
% as a result
will always return some ordering of representations.
When the evaluation dataset is too small or all of the representations are poor, it may be that no representation can yet solve the task (i.e. achieve a useful accuracy).
Since the order of representations can change as more data is added, any judgement would be premature.
Indeed, there is often an implicit minimum requirement for the loss a representation should achieve to be considered meaningful.
As we show in the next section, our methods makes this requirement explicit.

% Ultimately we care about achieving high predictive accuracy on the given task.
% We would not even care about the rankings of representations if all gave terrible validation loss.
% That is, there is often an implicit minimum requirement for the validation loss a representation should achieve for it to be considered meaningful.
% As we will see in the next section, our methods makes this requirement explicit.

\section{Surplus description length \& $\eps$ sample complexity}

% All of the methods discussed above are properties of the data.
% Validation accuracy depends on the specific amount of data in the training set, description length measures the code length of a particular realization of the data, and mutual information measures the inherent uncertainty in the data distribution.
% Instead of measuring properties of the data, we consider the task of measuring the complexity of the \emph{labeling function}.

% Ultimately, we care about achieving high predictive accuracy on the given task. We would not even care about the rankings of representations if all gave terrible validation accuracy! That is, there is often an implicit minimum requirement for the validation accuracy a representation should achieve for it to be considered meaningful in its own right. Our framework makes this requirement explicit by first specifying a target validation accuracy ($\eps$-error) and evaluating representations according to how easy (in terms of number of samples or surplus error) it is for the learning procedure $\mathcal{A}$ to achieve it.

The methods discussed above measure a property of the data, such as the attainable accuracy on $n$ points, by learning an unspecified function.
Instead, we propose to precisely define the function of interest and measure its complexity using data.
Fundamentally, we shift from making a statement about the inputs of an algorithm, like VA and MDL do, to a statement about the outputs.
% As a starting point for defining this function of interest, we introduce the idea of a \emph{labeling function}.

% \begin{definition}[Labeling function]
%     The labeling function is the deterministic function which maps from any realization of X to the corresponding label Y.
% \end{definition}

% Note that in general, this labeling function is not uniquely identifiable from data; when $H(\rmY \mid \rmX) \ne 0$, we consider that our data has noise.
% In this case it is only possible to predict each label, and thus identify the labeling function itself, up to a certain amount of uncertainty.
% Therefore we propose to measure the complexity of an approximate labeling function which differs from the true labeling function by a loss of at most $\eps$.
% In particular we are interested in measuring the \emph{learning complexity} of this approximate labeling function; that is, the cost of approximating this function using a specified learning algorithm $\mathcal{A}$.

% \begin{definition}[Learning complexity]
%     Given a data distribution $\mathcal{D}$, a learning algorithm $\mathcal{A}$, and a tolerance $\eps$.
%     % Let $L_{\mathcal{A}}(n) := L(\mathcal{A}, n)$.
%     A complexity measure $\tau$ is a function of the loss-data curve $L(\mathcal{A}, n)$ which measures the cost of learning an $\eps$-loss predictor on $\mathcal{D}$ using algorithm $\mathcal{A}$.
%     % \begin{align}

%     % \end{align}
% \end{definition}

% In the following section we introduce two related notions of this complexity: surplus description length and sample complexity.

\subsection{Surplus description length (SDL)}
Imagine trying to efficiently encode a large number of samples of a random variable $\rve$ which takes values in $\{1 \ldots K\}$ with probability $p(\rve)$.
An optimal code for these events has expected length\footnote{in nats} $\E[\ell(\rve)] = \E_{\rve} [ - \log p(\rve)] = H(\rve)$.
If this data is instead encoded using a probability distribution $\hat p$, the expected length becomes $H(\rve) + \KL{p}{\hat p}$.
We call $\KL{p}{\hat p}$ the \emph{surplus description length} (SDL) from encoding according to $\hat p$ instead of $p$:
\begin{align}
    \KL{p}{\hat p}
    % = \E_{\rve \sim p} \left[ \log \frac{p(\rve)}{\hat p(\rve)} \right]
    = \E_{\rve \sim p} \left[ \log p(\rve) - \log \hat p(\rve) \right].
\end{align}
When the true distribution $p$ is a delta, the entire length of a code under $\hat p$ is surplus since $\log 1 = 0$.

Recall that the prequential code for estimating MDL computes the description length of the labels given observations in an evaluation dataset by iteratively creating tighter approximations $\hat p_{1} \ldots \hat p_{n}$ and integrating the area under the curve.
Examining \Cref{eq:mdl_expected}, we see that
\begin{align}
     m_{\mathrm{MDL}}(\phi, \mathcal{D}, \mathcal{A},n) = \sum_{i=1}^n L(\mathcal{A}_\phi, i) \geq \sum_{i=1}^n H(\rmY \mid \phi(\rmX)).
\end{align}
%the expected description length of a dataset of size $n$ is
% \begin{align}
%     \E \left[ \ell(\rmY^n \mid \rmX^n) \right] &= \sum_{i=1}^n \E_{\rmX, \rmY} - \log q_i(\rmY \mid \rmX) = \sum_{i=1}^n L(\mathcal{A}, i)
% \end{align}

If $H(\rmY \mid \phi(\rmX)) > 0$, MDL grows without bound as the size of the evaluation dataset $n$ increases.

Instead, we propose to measure the complexity of a learned predictor $p(\rmY \mid \phi(\rmX))$ by computing the surplus description length of encoding an infinite stream of data according to the online code instead of the true conditional distribution.
\begin{definition}[Surplus description length of online codes] \label{def:sdl_entropy}
    Given random variables $\rmX, \rmY \sim \mathcal{D}$, a representation function $ \phi$, and a learning algorithm $\mathcal{A}$, define
    \begin{align}
        m_{\mathrm{SDL}}(\phi, \mathcal{D}, \mathcal{A}) &=
        %\sum_{i=1}^\infty \E_{\rmX, \rmY} \Big[- \log \hat{p}_i(\rmY \mid \phi(\rmX)) +  \log p(\rmY \mid \rmX)  \Big] %\\
        \sum_{i=1}^\infty \Big[ L(\mathcal{A}_\phi, i) - H(\rmY \mid \rmX) \Big].
    \end{align}
\end{definition}
This surplus description length beyond the optimal code improves on MDL by being bounded.
However, as discussed in \Cref{sec:mutual_info} entropy is intractible to estimate, and in practice it is more relevant to measure the cost of learning a good enough predictor rather than a theoretically perfect one.

We generalize this definition to measure the complexity of learning an approximating conditional distribution with loss $\eps$.
This corresponds to the additional description length incurred by encoding data with the learning algorithm $\mathcal{A}$ rather than using a fixed predictor with loss $\eps$.
% \footnote{Note that $\eps$ must be greater than $H(\rmY \mid \rmX)$ for this measure to be defined.}
% rather than the true conditional distribution only:
\begin{definition}[Surplus description length of online codes with a specified baseline] \label{def:sdl}
    Take random variables $\rmX, \rmY \sim \mathcal{D}$, a representation function $ \phi$, a learning algorithm $\mathcal{A}$, and a loss tolerance $\eps \ge H(\rmY \mid \rmX)$. Let $[c]_+$ denote $\max(0, c)$ and then we define
    \begin{align}
        m_{\mathrm{SDL}}(\phi, \mathcal{D}, \mathcal{A},\eps) &= \sum_{i=1}^\infty \Big[ L(\mathcal{A}_\phi, i) - \eps \Big]_+.
    \end{align}
\end{definition}
One interpretation of this measure is that it gives the cost (in terms of information) for re-creating an $\eps$-loss predictor when using the representation $\phi$.

In our framework, the surplus description length corresponds to computing the area between the loss-data curve and a baseline set by $y = \eps$.
Whereas MDL measures the complexity of a sample of $n$ points, SDL measures the complexity of a function which solves the task to $\eps$ tolerance.

\paragraph{Estimating the SDL.}
Naively computing SDL would require unbounded data and the estimation of $L(\mathcal{A}_\phi, i)$ for every $i$.
However, any reasonable learning algorithm obtains a better-generalizing predictor when given more i.i.d. data from the target distribution \citep{Kaplan2020ScalingLF}.
If we assume that algorithms are monotonically improving so that $L(\mathcal{A}, i+1) \le L(\mathcal{A}, i)$, SDL only depends on $i$ up to the first point where $L(\mathcal{A}, n) \le \eps$.
Approximating this integral can be done efficiently by taking a log-uniform partition of the evaluation dataset size and computing the Riemann sum as in \citet{Voita2020InformationTheoreticPW}.
Note that evaluating a representation only requires training probes, not the large representation functions themselves, and thus has modest computational requirements.
Crucially, if the tolerance $\eps$ is set unrealizeably low or the amount of available data is insufficient, an implementation is able to report that the given complexity estimate is only a lower bound.
In Appendix \ref{sec:sdl_details} we provide a detailed algorithm for estimating SDL and a theorem proving its data requirements, and the supplement includes an implementation.
%as a function of the sample complexity and desired confidence.

\subsection{$\eps$ sample complexity ($\eps$SC)}

In addition to surplus description length we introduce a second, conceptually simpler measure of representation quality: $\eps$ sample complexity.

\begin{definition}[Sample complexity of an $\eps$-loss predictor] \label{def:sample_complexity}
    Given random variables $\rmX, \rmY \sim \mathcal{D}$, a representation function $ \phi$, a learning algorithm $\mathcal{A}$, and a loss tolerance $\eps \ge H(\rmY \mid \phi(\rmX))$, define
    \begin{align}
         m_{\eps\mathrm{SC}}(\phi, \mathcal{D}, \mathcal{A},\eps) &= \min \Big\{ n \in \mathbb{N} : L(\mathcal{A}_\phi, n) \leq \eps \Big\}.
    \end{align}
\end{definition}

The $\eps$ sample complexity measures the complexity of learning an $\eps$-loss predictor by the number of samples it takes
%a given algorithm
to find it.
This measure allows the comparison of two representations by first picking a target function to learn (via a setting of $\eps$), then measuring which representation enables learning that function with less data.

In our framework, sample complexity corresponds to taking a horizontal slice of the loss-data curve at $y = \eps$, analogous to VA's slice at $y=n$.
VA makes a statement about the data (by setting $n$) and reports the accuracy of some function given that data.
In contrast, $\eps$ sample complexity specifies the desired function and determines its complexity by how many samples are needed to learn it.

\paragraph{Estimating the $\eps$SC.}
Given an assumption that algorithms are monotonically improving such that $ L(\mathcal{A}, n+1) \le L(\mathcal{A}, n)$, $\eps$SC can be estimated efficiently.
With $n$ finite samples in the evaluation dataset, an algorithm may estimate $\eps$SC by splitting the data into $k$ uniform-sized bins and estimating $L(\mathcal{A}, \nicefrac{ik}{n})$ for $i \in \{1 \ldots k \}$.
By recursively performing this search on the interval which contains the transition from $L > \eps$ to $L < \eps$, we can rapidly reach a precise estimate or report that $m_{\eps\mathrm{SC}}(\phi, \mathcal{D}, \mathcal{A},\eps) > n$.
A more detailed examination of the algorithmic considerations of estimating $ \eps$SC is in Appendix \ref{sec:sc_details}, and an implementation is available in the supplement.

\paragraph{Using objectives other than negative log-likelihood.}
% \todo{this}
Our exposition of $\eps$SC uses negative log-likelihood for consistency with other methods, such as MDL, which require it.
However, it is straightforward to extend $\eps$SC to work with whatever objective function is desired under the assumption that said objective is monotone with increasing data when using algorithm $\mathcal{A}$.
A natural choice in many cases would be prediction accuracy, where a practitioner might target e.g. a 95\% accurate predictor.

\subsection{Setting $\eps$}

A value for the threshold $\eps$ corresponds to the set of $\eps$-loss predictors that a representation should make easy to learn.
% a representation should enable the algorithm to learn.
Choices of $\eps \ge H(\rmY \mid \rmX)$ represent attainable functions, while selecting $\eps < H(\rmY \mid \rmX)$ leads to unbounded SDL and $\eps$SC for any choice of the algorithm $\mathcal{A}$.

For evaluating representation learning methods in the research community, we recommend using SDL and establishing benchmarks which specify (1) a downstream task, in the form of an evaluation dataset; (2) a criterion for success, in the form of a setting of $\eps$; (3) a standard probing algorithm $\mathcal{A}$.
% A low SDL verifies that the target $\eps$-loss predictor has low complexity
% description length?
% given this representation.
The setting of $\eps$ can be done by training a large model on the raw representation of the full evaluation dataset and using its validation loss as $\eps$ when evaluating other representations.
This guarantees that $\eps \ge H(\rmY \mid \rmX)$ and the task is feasible with any representation at least as good as the raw data.
In turn, this ensures that SDL is bounded.

In practical applications, $\eps$ should be a part of the design specification for a system.
As an example, a practitioner might know that an object detection system with 80\% per-frame accuracy is sufficient and labels are expensive.
For this task, the best representation would be one which enables the most sample efficient learning of a predictor with error $\eps = 0.2$ using a 0~--~1 loss.

\section{Experiments} \label{sec:experiments}

We empirically show the behavior of VA, MDL, SDL, and $\eps$SC with two sets of experiments on real data.
These experiments have the following goals:
\begin{enumerate}
    \item Test whether the theoretical issue of sensitivity to evaluation dataset size for VA and MDL occurs in practice.
    \item Demonstrate that SDL and $\eps$SC produce lower bounds when insufficient data is available and concrete quantities otherwise.
    \item Evaluate whether the computation of SDL and $\eps$SC scales to large-scale tasks.
\end{enumerate}

\begin{figure}[t!]
\centering
\includegraphics[width=0.45\textwidth]{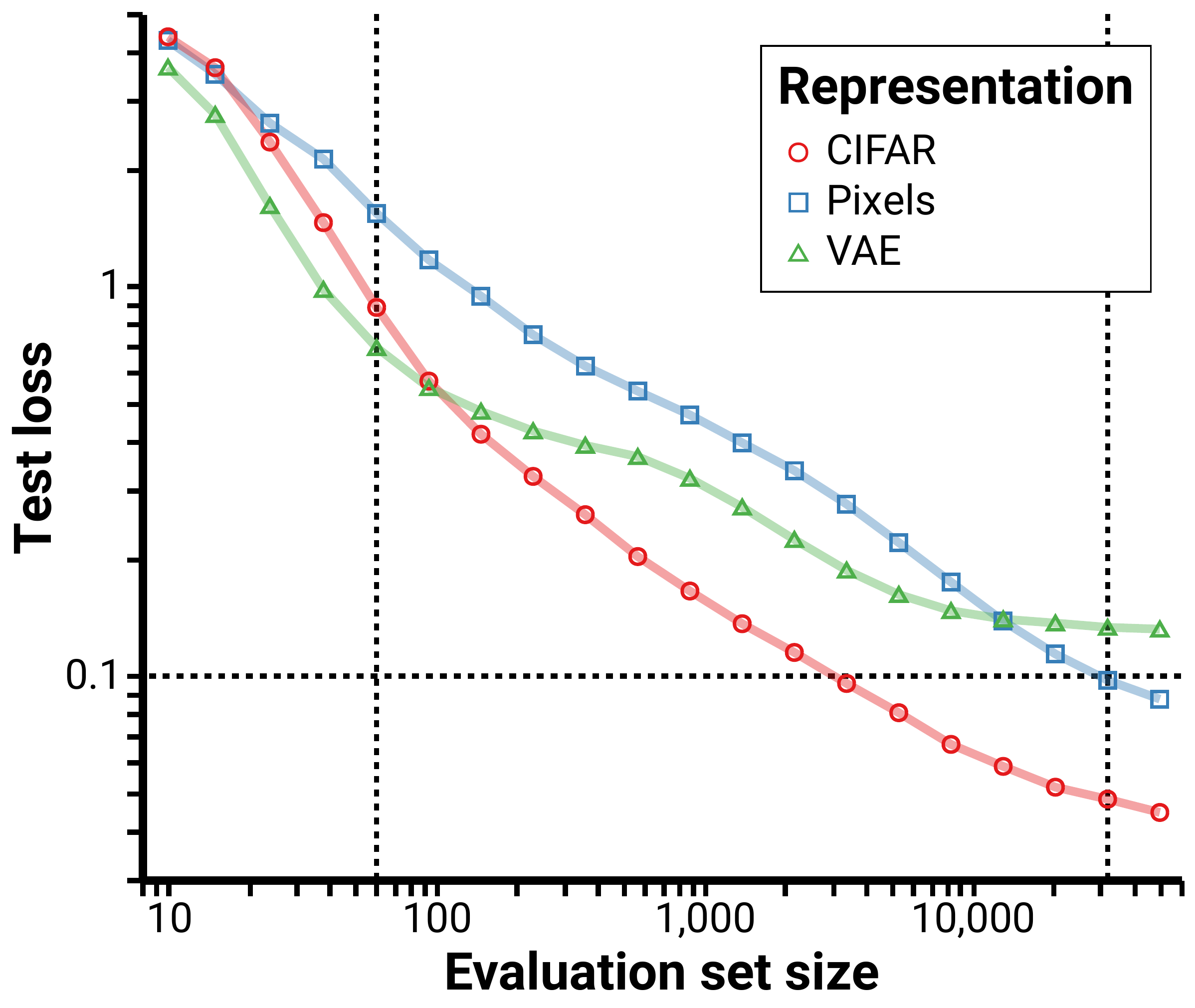}
  \caption{Results using three representations on MNIST.
  The intersections between curves indicate evaluation dataset sizes where VA would change its ranking of these representations. Curves are estimated using eight bootstrap-sampled evaluation datasets and initializations at each point to ensure the measured quantities are close to the expectation.}
  \label{fig:multi_mnist}
\end{figure}

\begin{table}[th]
    \centering
    {\small
\begin{tabular}{llabc}
\toprule
      & Representation &     CIFAR &    Pixels &      VAE \\
n & {} &           &           &          \\
\midrule
60    & VA &      0.88 &      1.54 &       \textbf{0.70} \\
      & MDL &    122.75 &    147.34 &       \textbf{93.8} \\
      & SDL, $\varepsilon$=0.1 &  > 116.75 &  > 141.34 &     > 87.8 \\
      & $\varepsilon$SC, $\varepsilon$=0.1 &    > 60.0 &    > 60.0 &     > 60.0 \\
31936 & VA &      \textbf{0.05} &      0.10 &       0.13 \\
      & MDL &    \textbf{2165.1} &   5001.57 &    4898.37 \\
      & SDL, $\varepsilon$=0.1 &     \textbf{260.6} &   1837.08 &  > 1704.77 \\
      & $\varepsilon$SC, $\varepsilon$=0.1 &      \textbf{3395} &     31936 &  > 31936.0 \\

% 60    & VA &      0.88 &      1.54 &     \textbf{0.70} \\
%       & MDL &    122.75 &    147.34 &     \textbf{93.8} \\
    %   & SDL, $\varepsilon$=1 &     65.33 &   > 87.34 &    \textbf{40.75} \\
    %   & SDL, $\varepsilon$=0.2 &  > 110.75 &  > 135.34 &   > 81.8 \\
    %   & $\varepsilon$SC, $\varepsilon$=1 &        60 &    > 60.0 &       \textbf{38} \\
    %   & $\varepsilon$SC, $\varepsilon$=0.2 &    > 60.0 &    > 60.0 &   > 60.0 \\
% 20398 & VA &      \textbf{0.05} &      0.11 &     0.14 \\
%       & MDL &   \textbf{1607.95} &   3876.88 &  3360.49 \\
    %   & SDL, $\varepsilon$=1 &     65.33 &     93.06 &    \textbf{40.75} \\
    %   & SDL, $\varepsilon$=0.2 &    \textbf{153.49} &    800.16 &   278.95 \\
    %   & $\varepsilon$SC, $\varepsilon$=1 &        60 &       147 &       \textbf{38} \\
    %   & $\varepsilon$SC, $\varepsilon$=0.2 &       \textbf{884} &      8322 &     3395 \\

\bottomrule
\end{tabular}
}
    \caption{Estimated measures of representation quality on MNIST. At small evaluation dataset sizes, VA and MDL state that the VAE representation is the best, even though every representation yields poor prediction quality with that amount of data. Since SDL and $\eps$SC have a target for prediction quality, they are able to report when the evaluation dataset is insufficient to achieve the desired performance.}
    \label{tab:multi_mnist}
\end{table}

\subsection{Tasks and representations}
For the first experiment, shown in \Cref{fig:multi_mnist} and \Cref{tab:multi_mnist}, we use the small-scale task of MNIST classification.
We evaluate three representations: (1) the last hidden layer of a small convolutional network pretrained on CIFAR-10; (2) the raw pixels; and (3) the bottleneck of a variational autoencoder (VAE) \citep{Kingma2014AutoEncodingVB,Rezende2014StochasticBA} trained on MNIST.

For the second experiment, shown in \Cref{fig:elmo_layers} and \Cref{tab:elmo_layers}, we compare the representations given by different layers of a pretrained ELMo model \citep{Peters2018DeepCW}.
We use the part-of-speech task introduced by \citet{Hewitt2019DesigningProbes} and implemented by \citet{Voita2020InformationTheoreticPW} with the same probe architecture and other hyperparameters as those works.
This leads to a large-scale representation evaluation task, with 4096-dimensional representation vectors and an output space of size $48^k$ for a sentence of $k$ words.

In each set of experiments we compute loss-data curves by estimating the expected population loss at each evaluation dataset size using a bootstrapped sample from the full evaluation dataset, reducing the variance of the results.
Note that in each experiment we omit MI as for a finite evaluation dataset, the MI measure is the same as validation loss.
Details of the experiments, including representation training, probe architectures, and hyperparameters, are available in Appendix \ref{sec:experiment_details}.

\subsection{Results}
These experiments demonstrate that the issue of sensitivity to evaluation dataset size in fact occurs in practice, both on small problems (\Cref{tab:multi_mnist}) and at scale (\Cref{tab:elmo_layers}): VA and MDL choose different representations when given evaluation sets of different sizes.
Because these measures are a function of the evaluation dataset size, making a decision about which representation to use with a small evaluation dataset would be premature.

By contrast, SDL and $\eps$SC are functions only of the data \emph{distribution}, not a finite sample.
Once they measure the complexity of learning an $\eps$-loss function, that measure is invariant to the size of the evaluation dataset.
Crucially, since these measures contain a notion of success in  solving a task, they are able to avoid the issue of premature evaluation and notify the user if there is insufficient data to evaluate and return a lower bound instead.

The part of speech experiment in \Cref{fig:elmo_layers} and \Cref{tab:elmo_layers} demonstrates that SDL and $\eps$SC can scale to tasks of a practically relevant size.
This experiment is of a similar size to the widespread use of BERT \citep{devlin2018bert} or SimCLR \citep{Chen2020ASF}, and evaluating our measures to high precision took about an hour on one GPU.

% These experiments also bear out the intuition that more precise functions (represented by smaller $\eps$) are more complex to learn than simpler functions (larger $\eps$).
% By making this dependence explicit, practicioners who need the highest-accuracy models can directly evaluate the cost of learning such models

\begin{figure}[tb]
\centering
\includegraphics[width=0.45\textwidth]{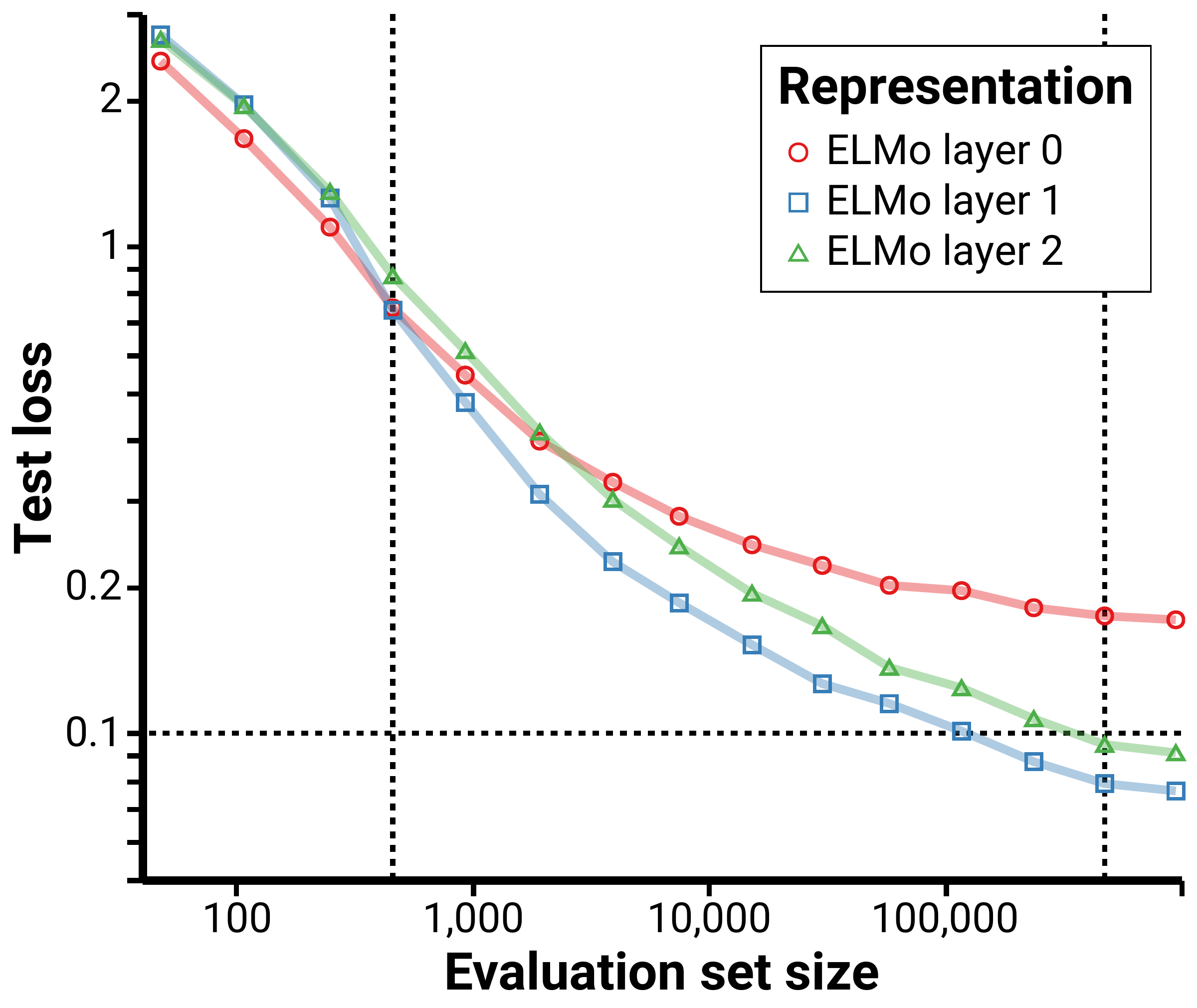}
\caption{Results using three representations on the part of speech classification task. Loss-data curves are estimated using four bootstrap-sampled evaluation datasets and network initializations at each point.}
\label{fig:elmo_layers}
\end{figure}

\begin{table}[tb]
    \centering
    {\small
\begin{tabular}{llabc}
\toprule
       & ELMo layer &           0 &         1 &         2 \\
n & {} &             &           &           \\
\midrule
% 461    & VA &        0.73 &      \textbf{0.72} &      0.85 \\
%       & MDL &     \textbf{1213.98} &   1313.75 &   1320.52 \\
%       & SDL, $\varepsilon$=0.5 &    > 283.75 &  > 334.43 &  > 366.35 \\
%       & SDL, $\varepsilon$=0.1 &    > 472.15 &  > 522.83 &  > 554.75 \\
%       & $\varepsilon$SC, $\varepsilon$=0.5 &       > 461 &     > 461 &     > 461 \\
%       & $\varepsilon$SC, $\varepsilon$=0.1 &       > 461 &     > 461 &     > 461 \\
% 474838 & VA &        0.17 &      \textbf{0.08} &      0.09 \\
%       & MDL &    92403.41 &  \textbf{52648.50} &  65468.54 \\
%       & SDL, $\varepsilon$=0.5 &      \textbf{338.48} &    365.25 &    450.84 \\
%       & SDL, $\varepsilon$=0.1 &  > 42162.47 &   \textbf{2922.04} &   8734.11 \\
%       & $\varepsilon$SC, $\varepsilon$=0.5 &        1256 &       \textbf{854} &      1446 \\
%       & $\varepsilon$SC, $\varepsilon$=0.1 &    > 474838 &    \textbf{149946} &    474838 \\
461    & VA &        0.75 &      \textbf{0.74} &     0.87 \\
       & MDL &      \textbf{884.54} &   1009.26 &  1017.72 \\
       & SDL, $\varepsilon$=0.1 &    > 478.67 &  > 528.51 &  > 561.7 \\
       & $\varepsilon$SC, $\varepsilon$=0.1 &       > 461 &     > 461 &    > 461 \\
474838 & VA &        0.17 &      \textbf{0.08} &     0.09 \\
       & MDL &    92403.41 &  \textbf{52648.50} & 65468.54 \\
       & SDL, $\varepsilon$=0.1 &  > 40882.72 &   \textbf{2765.11} &  7069.56 \\
       & $\varepsilon$SC, $\varepsilon$=0.1 &    > 474838 &    \textbf{237967} &   474838 \\

\bottomrule
\end{tabular}
}
    \caption{Estimated measures of representation quality on the part of speech classification task. With small evaluation datasets, MDL finds that the lowest ELMo layer gives the best results, but when the evaluation dataset grows the outcome changes.}
    \label{tab:elmo_layers}
\end{table}

% \begin{figure}[h]
% \centering
% \includegraphics[width=0.45\textwidth]{figures/elmo_layers_loss.png}
% \caption{
% Here we have the loss-data curves for the part of speech task using representation given by different layers of a pretrained ELMo model.
% Crossover is clearly visible.
% }
% \label{fig:elmo_layers}
% \end{figure}

\section{Related work}

\paragraph{Representation evaluation methods.}
Until recently, the standard technique for evaluating representation quality was the use of linear probes \citep{Kiros2015SkipThoughtV,Hill2016LearningDR,Oord2018RepresentationLW,Chen2020ASF}.
However, \citet{Hnaff2020DataEfficientIR} find that evaluation with linear probes is largely uncorrelated with the more practically relevant objective of low-data accuracy, and \citet{resnick2019probing} show that linear probe performance does not predict performance for transfer across tasks.
% Moving beyond linear probes, \citet{Zhang2018LanguageMT} and \citet{Hewitt2019DesigningProbes} propose random baselines for linguistic tasks to provide context for how much linguistic structure is readily accessible in representations.
% To show separation between the validation accuracy achieved by these random baselines and representations pretrained on genuine linguistic labels, they have to limit the amount of training data or restrict the capacity of probes.
Beyond linear probes, \citet{Zhang2018LanguageMT} and \citet{Hewitt2019DesigningProbes} show that restrictions on model capacity or evaluation dataset size are necessary to separate the performance of randomly- and linguistically-pretrained representations.
\citet{Voita2020InformationTheoreticPW} propose using the MDL framework, which
% accounts for the ``effort of learning'' required by the probes to achieve high validation accuracy,
measures the description length of the labels given the observations.
% , and show that this measure separates these representations without capacity restrictions.
An earlier work by~\citet{Yogatama2019LinguisticIntel} also uses prequential codes to evaluate representations for linguistic tasks.
\citet{Talmor2019oLMpics} look at the loss-data curve (called ``learning curve'' in their work) and use a weighted average of the validation loss at various training set sizes to evaluate representations.
% Unlike this work, none of these prior methods address the motivation of minimizing

\paragraph{Foundational work.}
A fundamental paper by \citet{Blier2018TheDL} introduces prequential codes as a measure of the complexity of a deep learning model.
\citet{Xu2020ATheory} introduce predictive $\mathcal{V}$-information, a theoretical generalization of mutual information which takes into account computational constraints,
and is essentially the mutual information lower bound often reported in practice.
% and their $\mathcal{V}$-information is equivalent to the mutual information lower bound often used in practice..
% Their work describes the setting with a learning algorithm searching within a family of functions for the best predictive model; their $\mathcal{V}$-information is equivalent to the mutual information lower bound often used in practice.
Work by \citet{Dubois2020LearningOR} describe representations which, in combination with a specified family of predictive functions, have guarantees on their generalization performance.

\section{Discussion}

In this work, we have introduced the loss-data framework for comparing representation evaluation measures and used it to diagnose the issue of sensitivity to evaluation dataset size in the validation accuracy and minimum description length measures.
We proposed two measures, surplus description length and $\eps$ sample complexity, which eliminate this issue by measuring the complexity of learning a predictor which solves the task of interest to $\eps$ tolerance.
Empirically, we showed that sensitivity to evaluation dataset size occurs in practice for VA and MDL, while SDL and $\eps$SC are robust to the amount of available data and are able to report when it is insufficient to make a judgment.

Each of these measures depends on a choice of algorithm $\mathcal{A}$, including hyperparameters such as probe architecture, which could make the evaluation procedure less robust.
To alleviate this, future work might consider a set of algorithms $A = \{\mathcal{A}_i\}_{i=1}^K$ and a method of combining them, such as the model switching technique of \citet{Blier2018TheDL,Erven2012CatchingUF} or a Bayesian prior.

Finally, while existing measures such as VA, MI, and MDL do not measure our notion of the best representation for a task, under other settings they may be the correct choice.
For example, if only a fixed set of data will ever be available, selecting representations using VA might be a reasonable choice; and if unbounded data is available for free, perhaps MI is the most appropriate measure.
However, in many cases the robustness and interpretability offered by SDL and $\eps$SC make them a practical choice for practitioners and representation researchers alike.

% \paragraph{Extending to multiple probing algorithms.}
% \todo{rewrite this paragraph}
% The reliance on the choice of algorithm $ \mathcal{A}$ makes the evaluation procedure less robust.
% To alleviate this somewhat, we can instead consider a set of algorithms $ A = \{\mathcal{A}_i\}_{i=1}^K$.
% Each $ \mathcal{A}_i$ may be a different learning algorithm, different probe architecture, or different setting of hyperparameters.
% Now we just need to choose how define \emph{easiness} with respect to a set of algorithms.
% We could choose something like the mean, median, or minimum of either measure over the algorithms in the set.
% We can think about our choice of the set of algorithms as representing our prior beliefs over how the representation is connected to the task.
% Specifying this set is just specifying our prior.

\newpage

\bibliographystyle{note}
\bibliography{references}
\newpage
\onecolumn

\begin{appendices}
\section{Algorithmic details for estimating surplus description length} \label{sec:sdl_details}

Recall that the SDL is defined as
\begin{align}
        m_{\mathrm{SDL}}(\phi, \mathcal{D}, \mathcal{A},\eps) &= \sum_{n=1}^\infty \Big[ L(\mathcal{A}_\phi, n) - \eps \Big]_+
\end{align}
For simplicity, we assume that $L$ is bounded in $[0,1]$. Note that this can be achieved by truncating the cross-entropy loss.

\begin{algorithm}
\caption{Estimate surplus error}
\label{alg:sdl}
\KwIn{tolerance $\eps $, max iterations $M$, number of datasets $K$, representation $ \phi$, data distribution $\mathcal{D}$, algorithm $\mathcal{A}$ }
\KwOut{Estimate $\hat m $ of  $m(\phi, \mathcal{D}, \eps, \mathcal{A})$ and indicator $ I$ of whether this estimate is tight or lower bound}
\vspace{1mm} \hrule \vspace{1mm}
Sample $ K$ datasets $ D_M^{k}\sim \mathcal{D}$ of size $ M+1$\\
\For{$n = 1$ \KwTo $M$}{
    For each $ k \in [K]$, run $ \mathcal{A}$ on $ D_M^{k}[1:n]$ to produce a predictor $ \hat p_n^k$\\
    Take $ K $ test samples $ (x_k, y_k) = D_M^k[M+1]$\\
    Evaluate $ \hat L_n = \frac{1}{K}\sum_{k=1}^K \ell(\hat p_n^k, x_k , y_k) $
    }
Set $ \hat m = \sum_{n=1}^M [\hat L_n - \eps]_+$ \vspace{1mm}\\
\lIf {$ \hat L_M \leq \eps/2$} {Set $I = $ \texttt{tight} \textbf{else} {Set $ I = $ \texttt{lower bound}}}
\Return $\hat m, I$
\end{algorithm}

In our experiments we replace $D^k_M[1:n]$ of Algorithm \ref{alg:sdl} with sampled subsets of size $n$ from a single evaluation dataset.
Additionally, we use between 10 and 20 values of $n$ instead of evaluating $L(\mathcal{A}_\phi, n)$ at every integer between $1$ and $M$.
This strategy, also used by \citet{Blier2018TheDL} and \citet{Voita2020InformationTheoreticPW}, corresponds to the description length under a code which updates only periodically during transmission of the data instead of after every single point.

\begin{theorem}
Let the loss function $L$ be bounded in $[0,1]$ and assume that it is decreasing in $ n$. With $ (M+1)K $ datapoints, if the sample complexity is less than $ M$, the above algorithm returns an estimate $ \hat m$ such that with probability at least $ 1- \delta$
\begin{align}
    |\hat m - m(\phi, \mathcal{D}, \eps, \mathcal{A})| \leq  M\sqrt{ \frac{\log (2M/\delta)}{2K}}.
\end{align}
If $ K \geq \frac{\log(1/\delta)}{2\eps^2}$ and the algorithm returns \texttt{tight} then with probability at least $ 1-\delta$ the sample complexity is less than $ M $ and the above bound holds.
\end{theorem}
\begin{proof}
First we apply a Hoeffding bound to show that each $ \hat L_n$ is estimated well. For any $ n$, we have
\begin{align}
    P \bigg( \big|\hat L_n   - L(\mathcal{A}_\phi,n)  \big| > \sqrt{\frac{\log(2M/\delta)}{2K}} \bigg) \leq 2 \exp\bigg(-2K  \frac{\log(2M/\delta)}{2K}\bigg) = 2 \frac{\delta}{2M} = \frac{\delta}{M}
\end{align}
since each $ \ell(\hat p_n^k, x_k , y_k)$ is an independent variable, bounded in [0,1] with expectation $ L(\mathcal{A}_\phi, n)$.

Now when sample complexity is less than $ M$, we use a union bound to translate this to a high probability bound on error of $ \hat m$, so that with probability at least $ 1- \delta$:
\begin{align}
    |\hat m - m(\phi, \mathcal{D}, \eps, \mathcal{A})| &= \bigg|\sum_{n=1}^M [\hat L_n - \eps]_+  - [L(\mathcal{A}_\phi,n) - \eps]_+  \bigg|\\
    &\leq \sum_{n=1}^M\bigg| [\hat L_n - \eps]_+  - [L(\mathcal{A}_\phi,n) - \eps]_+ \bigg|\\
    &\leq \sum_{n=1}^M \bigg|\hat L_n - L(\mathcal{A}_\phi,n) \bigg|\\
    &\leq M \sqrt{ \frac{\log (2M/\delta)}{2K}}
\end{align}
This gives us the first part of the claim.

We want to know that when the algorithm returns \texttt{tight}, the estimate can be trusted (i.e. that we set $ M $ large enough). Under the assumption of large enough $K$, and by an application of Hoeffding, we have that
\begin{align}
    P \bigg(  L(\mathcal{A}_\phi,M) - \hat L_M  > \eps/2 \bigg) \leq  \exp\bigg(-2K \eps^2 \bigg) \leq  \exp\bigg(-2 \frac{\log(1/\delta)}{2\eps^2} \eps^2 \bigg) = \delta
\end{align}
If $ \hat L_M \leq \eps/2$, this means that $ L(\mathcal{A}_\phi,M) \leq \eps$ with probability at least $ 1-\delta$. By the assumption of decreasing loss, this means the sample complexity is less than $ M$, so the bound on the error of $ \hat m$ holds.
\end{proof}

\section{Algorithmic details for estimating sample complexity} \label{sec:sc_details}
Recall that $\eps$ sample complexity ($\eps$SC) is defined as
\begin{align}
     m_{\eps\mathrm{SC}}(\phi, \mathcal{D}, \mathcal{A},\eps) &= \min \Big\{ n \in \mathbb{N} : L(\mathcal{A}_\phi, n) \leq \eps \Big\}.
\end{align}

We estimate $m_{\eps\mathrm{SC}}$ via recursive grid search. To be more precise, we first define a search interval $[1,N]$, where $N$ is a large enough number such that $L(\mathcal{A}_\phi,N) \ll \eps$. Then, we partition the search interval in to 10 sub-intervals and estimate risk of hypothesis learned from $D^n \sim \mathcal{D}^n$ with high confidence for each sub-interval. We then find the leftmost sub-interval that potentially contains $m_{\eps\mathrm{SC}}$ and proceed recursively. This procedure is formalized in Algorithm~\ref{alg:esc} and its guarantee is given by Theorem~\ref{thm:esc}.
\begin{algorithm}[h!]
\caption{Estimate sample complexity via recursive grid search}
\label{alg:esc}
\KwIn{Search upper limit $N$, parameters $\eps$, confidence parameter $\delta$, data distribution $\mathcal{D}$, and learning algorithm $\mathcal{A}$.}
\KwOut{Estimate $\hat{m}$ such that $m_{\eps\mathrm{SC}}(\phi,\mathcal{D},\mathcal{A},\eps) \le \hat m$ with probability $1-\delta$.}
\vspace{1mm} \hrule \vspace{1mm}
let $S = 2\log (20k/\delta)/\eps^2$, and let $[\ell,u]$ be the search interval initialized at $\ell = 1, u = N$.\\
\For{$r=1$ \KwTo $k$}{
    Partition $[\ell,u]$ into 10 equispaced bins and let $\Delta$ be the length of each bin. \\
    \For{$j = 1$ \KwTo $10$}{
        Set $n = \ell + j \Delta$. \\
        Compute $\hat L_n = \frac{1}{S}\sum_{i=1}^S \ell(\mathcal{A}(D^n_i),x_i,y_i)$ for $S$ independent draws of $D^n$ and test sample $(x,y)$. \\
        \If{$\hat L_n \le \eps/2$}{
        Set $u = n$ and $\ell = n - \Delta$. \\
        \textbf{break}
        }
        }
}
\Return $\hat m = u$, which satisfies $m_{\eps\mathrm{SC}}(\phi,\mathcal{D},\mathcal{A},\eps) \le \hat m$ with probability $1-\delta$, where the randomness is over independent draws of $D^n$ and test samples $(x,y)$.
\end{algorithm}

\begin{theorem}
\label{thm:esc}
Let the loss function $L$ be bounded in $[0,1]$ and assume that it is decreasing in $ n$. Then, Algorithm~\ref{alg:esc} returns an estimate $ \hat m$ that satisfies $m_{\eps\mathrm{SC}}(\phi,\mathcal{D},\mathcal{A},\eps) \le \hat m$ with probability at least $ 1- \delta$.
\end{theorem}

\begin{proof}
By Hoeffding, the probability that $|\hat L_n-L(\mathcal{A}_{\phi},n)| \ge \eps/2$, where $\hat L$ is computed with $S = 2\log(20k/\delta)/\eps^2$ independent draws of $D^n \sim \mathcal{D}^n$ and $(x,y) \sim \mathcal{D}$, is less than $\delta/(10k)$. The algorithm terminates after evaluating $\hat L$ on at most $10k$ different $n$'s. By a union bound, the probability that $|\hat L_n - L(\mathcal{A}_{\phi},n)| \le \eps/2$ for all $n$ used by the algorithm is at least $1-\delta$. Hence, $\hat L_n \le \eps/2$ implies $L(\mathcal{A}_\phi,n) \le \eps$ with probability at least $1-\delta$.
\end{proof}

\section{Experimental details} \label{sec:experiment_details}

In each experiment we first estimate the loss-data curve using a fixed number of dataset sizes $n$ and multiple random seeds, then compute each measure from that curve.
Reported values of SDL correspond to the estimated area between the loss-data curve and the line $y=\eps$ using Riemann sums with the values taken from the left edge of the interval.
% , which can be interpreted as transmitting the targets $Y^n$ using a code with these pre-specified transition points between models.
This is the same as the chunking procedure of \citet{Voita2020InformationTheoreticPW} and is equivalent to the code length of transmitting each chunk of data using a fixed model and switching models between intervals.
Reported values of $\eps$SC correspond to the first measured $n$ at which the loss is less than $\eps$.

All of the experiments were performed on a single server with 4 NVidia Titan X GPUs, and on this hardware no experiment took longer than an hour.
All of the code for our experiments, as well as that used to generate our plots and tables, is included in the supplement.

\subsection{MNIST experiments}

For our experiments on MNIST, we implement a highly-performant vectorized library in \hyperlink{https://jax.readthedocs.io/en/latest/}{JAX} to construct loss-data curves.
With this implementation it takes about one minute to estimate the loss-data curve with one sample at each of 20 settings of $n$.
We approximate the loss-data curves at 20 settings of $n$ log-uniformly spaced on the interval $[10, 50000]$ and evaluate loss on the test set to approximate the population loss.
At each dataset size $n$ we perform the same number of updates to the model; we experimented with early stopping for smaller $n$ but found that it made no difference on this dataset.
In order to obtain lower-variance estimates of the expected risk at each $n$, we run 8 random seeds for each representation at each dataset size, where each random seed corresponds to a random initialization of the probe network and a random subsample of the evaluation dataset.

Probes consist of two-hidden-layer MLPs with hidden dimension 512 and ReLU activations.
All probes and representations are trained with the Adam optimizer \citep{Kingma2015AdamAM} with learning rate $10^{-4}$.

Each representation is normalized to have zero mean and unit variance before probing to ensure that differences in scaling and centering do not disrupt learning.
The representations of the data we evaluate are implemented as follows.

\paragraph{Raw pixels.}
The raw MNIST pixels are provided by the Pytorch \texttt{datasets} library \citep{Paszke2019PyTorchAI}.
It has dimension $28 \times 28 = 784$.

\paragraph{CIFAR.}
The CIFAR representation is given by the last hidden layer of a convolutional neural network trained on the CIFAR-10 dataset.
This representation has dimension 784 to match the size of the raw pixels.
The network architecture is as follows:

\begin{verbatim}
    nn.Conv2d(1, 32, 3, 1),
    nn.ReLU(),
    nn.MaxPool2d(2),
    nn.Conv2d(32, 64, 3, 1),
    nn.ReLU(),
    nn.MaxPool2d(2),
    nn.Flatten(),
    nn.Linear(1600, 784)
    nn.ReLU()
    nn.Linear(784, 10)
    nn.LogSoftmax()
\end{verbatim}

\paragraph{VAE.}
The VAE (variational autoencoder; \citet{Kingma2014AutoEncodingVB,Rezende2014StochasticBA}) representation is given by a variational autoencoder trained to generate the MNIST digits.
This VAE's latent variable has dimension 8.
We use the mean output of the encoder as the representation of the data.
The network architecture is as follows:
\begin{verbatim}
self.encoder_layers = nn.Sequential(
    nn.Linear(784, 400),
    nn.ReLU(),
    nn.Linear(400, 400),
    nn.ReLU(),
    nn.Linear(400, 400),
    nn.ReLU(),
)
self.mean = nn.Linear(400, 8)
self.variance = nn.Linear(400, 8)

self.decoder_layers = nn.Sequential(
    nn.Linear(8, 400),
    nn.ReLU(),
    nn.Linear(400, 400),
    nn.ReLU(),
    nn.Linear(400, 784),
)
\end{verbatim}

\subsection{Part of speech experiments}

We follow the methodology and use the official code\footnote{\url{https://github.com/lena-voita/description-length-probing}} of \citet{Voita2020InformationTheoreticPW} for our part of speech experiments using ELMo \citep{Peters2018DeepCW} pretrained representations.
In order to obtain lower-variance estimates of the expected risk at each $n$, we run 4 random seeds for each representation at each dataset size, where each random seed corresponds to a random initialization of the probe network and a random subsample of the evaluation dataset.
We approximate the loss-data curves at 10 settings of $n$ log-uniformly spaced on the range of the available data $n \in [10, 10^6]$.
To more precisely estimate $\eps$SC, we perform one recursive grid search step: we space 10 settings over the range which in the first round saw $L(\mathcal{A}_\phi, n)$ transition from above to below $\eps$.

Probes consist of the MLP-2 model of \citet{Hewitt2019DesigningProbes,Voita2020InformationTheoreticPW} and all training parameters are the same as in those works.
\end{appendices}

\end{document}